\documentclass{article}
\usepackage{amsmath,amssymb,graphicx,booktabs,mlspconf}
\usepackage{xcolor}
\usepackage{tikz}
\usetikzlibrary{arrows.meta,positioning,calc,fit,backgrounds,shapes.geometric,shadings}

\definecolor{cBack}{RGB}{64,83,112}
\definecolor{cAttn}{RGB}{0,114,178}
\definecolor{cPool}{RGB}{213,94,0}
\definecolor{cGraph}{RGB}{0,158,115}
\definecolor{cHead}{RGB}{120,94,160}
\definecolor{cHot}{RGB}{222,45,38}
\definecolor{cLite}{RGB}{236,238,242}
\definecolor{cGray}{RGB}{105,105,105}
\definecolor{cGCG}{RGB}{255,193,7}

\newtheorem{proposition}{Proposition}

\copyrightnotice{979-8-3503-2411-2/25/\$31.00 {\copyright}2025 IEEE}
\toappear{2026 IEEE International Workshop on Machine Learning for Signal Processing, Sep.\ 28-- Oct.\ 1, 2026, Atlanta, USA}

\title{C$^2$A: Coupling Spatial Evidence with Clinical Priors via Co-occurrence Aware Class Attention for Multi-Label Chest X-Ray Classification}

\name{Akash Gogineni$^{\star\dagger}$ \quad Nagur Shareef Shaik$^{\star}$ \quad Aasrith Mandava$^{\star}$ \quad Adnan Masood$^{\dagger}$ \quad Dong Hye Ye$^{\star}$}

\address{$^{\star}$ Department of Computer Science, Georgia State University, Atlanta, GA 30324, USA \\ $^{\dagger}$UST Global Inc., Aliso Viejo, CA 92656, USA}

\begin{document}
\maketitle


\begin{abstract}
Thoracic pathologies rarely occur in isolation, yet standard multi-label classifiers rely on shared global descriptors, discarding \emph{where} findings lie and \emph{how} they co-occur. We propose \textbf{C$\mathbf{^2}$A} (Co-occurrence Aware Class Attention), a classification head that explicitly couples spatial evidence with clinical priors. First, C$^2$A casts pooling as an expectation over learned per-class spatial attention maps, yielding localized descriptors for each disease. Second, it couples these descriptors via a learnable graph warm-started from empirical label co-occurrence. A single residual message-passing step shares evidence among related findings, proving to be a bounded perturbation of the identity where co-occurrence enters each logit through an explicit bilinear interaction. On CheXpert, C$^2$A achieves a superior $0.895$ macro-mean AUROC, outperforming advanced context-gating baselines. Crucially, gains concentrate on highly co-occurrent classes with ambiguous spatial evidence (rescuing Atelectasis by $+1.5$ over GCG), demonstrating the prior's regularizing effect with a negligible overhead of one linear projection and a $C\!\times\!C$ edge matrix.
\end{abstract}


\vspace{-10pt}
\section{Introduction}
\label{sec:intro}

The chest radiograph (CXR) is the most frequently acquired diagnostic image in the world~\cite{johnson2019mimic, wang2017chestx}. Automating its interpretation is intrinsically multi-label: a single study routinely contains several concurrent findings, requiring a model to emit an independent probability of presence for each pathology rather than a single class decision~\cite{irvin2019chexpert}. Crucially, these targets are not statistically independent. Thoracic findings are physiologically coupled: cardiogenic pulmonary edema develops in the setting of an enlarged cardiac silhouette, and lobar collapse shares its radiographic appearance with adjacent air-space opacity~\cite{pham2021hierarchical, chen2019mlgcn}. Radiologists read each finding in the context of the others, an ambiguous basal opacity is interpreted differently when the heart is enlarged than when it is not. However, standard classifiers score every finding from one shared, spatially-averaged feature vector~\cite{rajpurkar2017chexnet, huang2017densely}, a bottleneck known to dilute the localized visual evidence essential for fine-grained thoracic diagnosis~\cite{huang2021gloria}. This structurally prevents the model from utilizing clinical context and leaves it unable to indicate the specific image region supporting each decision.

Two largely separate lines of work have attempted to improve CXR classifiers, yet neither fully closes this gap. The first focuses on attention mechanisms to better localize findings. Methods augmenting convolutional backbones with spatial and channel attention (e.g., CBAM~\cite{woo2018cbam}, Triplet Attention~\cite{misra2021rotate}) or context-gating mechanisms (e.g., Global Context~\cite{cao2020global}, Guided Context Gating~\cite{cherukuri2024guided}) have become standard. However, these mechanisms typically operate on the shared feature map \emph{before} it is collapsed by global average pooling (GAP). Even in architectures that generate class-specific attention maps, the resulting descriptors remain isolated from one another prior to classification \cite{shaik2025ordinal, cherukuri2025dynamic}. They ignore the clinical reality that the presence of one disease alters the visual interpretation of another. The second line of work models label dependencies directly. Approaches like Multi-Label Graph Convolutional Networks~\cite{chen2019mlgcn} or hierarchical priors~\cite{pham2021hierarchical} explicitly encode co-occurrence. Yet, these methods typically construct graphs over abstract label embeddings or logit representations. Because these nodes are decoupled from localized spatial evidence, the network cannot perform spatially-grounded reasoning. Consequently, the two cues a radiologist actually combines, \emph{where} a finding is, and \emph{which} findings accompany it, are handled by disjoint machinery that never interacts.

We propose \textbf{C$\mathbf{^2}$A}, a lightweight classification head that explicitly couples spatial evidence with clinical priors. By replacing global pooling with class-specific spatial attention, C$^2$A extracts isolated descriptors for each pathology. These descriptors exchange context via a directed graph, warm-started from empirical label co-occurrence and refined via gradient descent. Interpretable by construction, the architecture reduces to standard baselines in well-defined limits. Crucially, we prove this coupling of expectation pooling and message passing is a bounded perturbation with an exact bilinear logit decomposition, achieving targeted diagnostic gains at negligible cost.

\begin{figure*}[t]
\centering
\includegraphics[width=\linewidth]{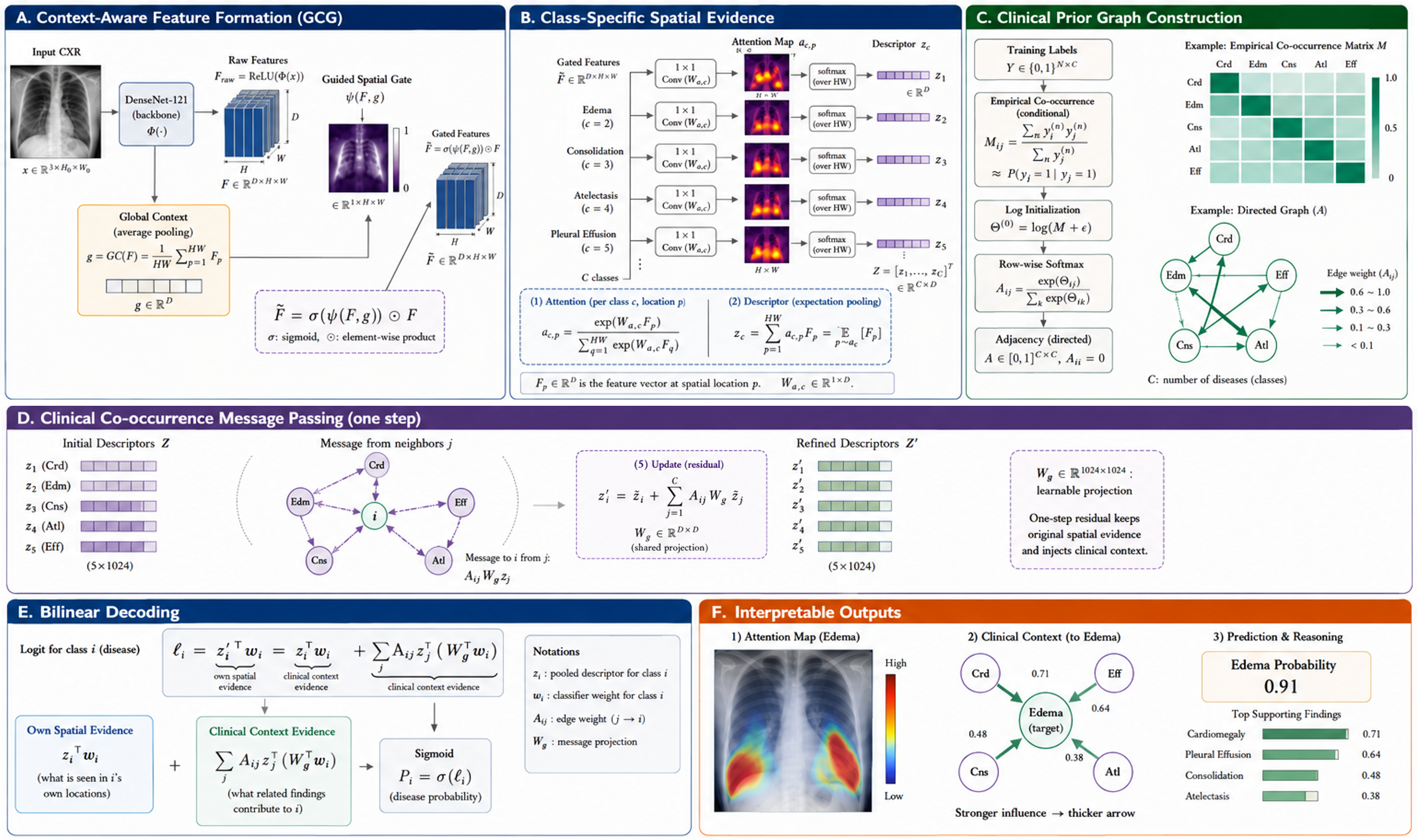}
\caption{\textbf{Co-occurrence-aware class-specific attention.} A chest radiograph ($3{\times}224{\times}224$) is encoded by a DenseNet-121 backbone into features $F\!\in\!\mathbb{R}^{1024\times7\times7}$. \emph{Guided Context Gating} (GCG) refines them by forming a global-context guidance $g{=}\mathrm{GC}(F)$ and applying a guided spatial gate, $\tilde F{=}\sigma(\psi(F,g))\odot F$. \emph{Class-specific attention} is then computed \emph{on the gated features}: a $1{\times}1$ convolution with a spatial softmax produces one map per pathology, $a\!\in\!\mathbb{R}^{5\times7\times7}$, and expectation pooling collapses each map into a class descriptor, stacked as $Z\!\in\!\mathbb{R}^{5\times1024}$. The descriptors are nodes of a directed graph whose adjacency $A$ is warm-started from the empirical conditional co-occurrence $M_{ij}{=}P(i|j)$ (right; darker $=$ larger, edge width $\propto$ co-occurrence); one residual message-passing step yields refined descriptors $Z'\!\in\!\mathbb{R}^{5\times1024}$, which a per-class linear classification network maps to the five disease probabilities $\hat y\!\in\![0,1]^{5}$. Only the head is added on top of the backbone.}
\label{fig:arch}
\end{figure*}

\section{Methodology}
\label{sec:method}

Given a dataset $\mathcal{D}=\{(x^{(n)},y^{(n)})\}_{n=1}^{N}$ of chest radiographs $x$ and multi-hot labels $y\in\{0,1\}^{C}$, we seek a mapping $f_\theta:x\mapsto\hat{y}\in[0,1]^{C}$ that emits an independent presence probability for each of the $C$ pathologies. The proposed \textbf{C$\mathbf{^2}$A} architecture achieves this through three sequential stages: extracting context-aware spatial features, decoupling these features into class-specific representations, and finally coupling them via a clinical co-occurrence graph.

\subsection{Context Aware Feature Formation}
\label{ssec:backbone}
We instantiate the feature extractor using a convolutional backbone $\Phi$, specifically DenseNet-121 \cite{huang2017densely} initialized with ImageNet weights. For a given input image $x$, the backbone produces a raw, high-dimensional feature tensor $F_{\text{raw}}=\mathrm{ReLU}(\Phi(x))\in\mathbb{R}^{D\times H\times W}$, where $D=1024$. 

Because thoracic pathologies often exhibit diffuse, global patterns (e.g., Cardiomegaly), local receptive fields alone are insufficient. To enrich the local descriptors with global image-level context, we process $F_{\text{raw}}$ through a Guided Context Gating (GCG) \cite{cherukuri2024guided} module. The GCG module first aggregates global context via a $1\times1$ convolutional self-attention bottleneck, producing a global guidance signal $g$. This signal is then used to spatially gate the original feature map via an attention gate $\psi$, yielding a context-aware feature tensor $\tilde{F} = \sigma(\psi(F_{\text{raw}}, g)) \odot F_{\text{raw}} \in \mathbb{R}^{D\times H\times W}$. Rather than collapsing this tensor globally, the C$^2$A head preserves its spatial resolution to allow for pathology-specific localization.

\subsection{Class-Specific Spatial Evidence}
\label{ssec:csa}
Standard architectures apply Global Average Pooling (GAP) across the spatial dimensions of $\tilde{F}$, which dilutes the signal of highly focal pathologies (e.g., a small lung nodule) by averaging them with healthy tissue. Inspired by recent advances in class-specific attention for multi-label recognition~\cite{liu2021query2label}, we isolate the spatial support of distinct findings by projecting the context-aware tensor $\tilde{F}$ via a $1\times1$ convolution $W_a\in\mathbb{R}^{C\times D}$ to obtain $C$ class-specific score maps. These are normalized spatially via a softmax to form probability distributions over the image:
\begin{equation}
a_{c,p}=\frac{\exp(W_{a,c} \tilde{F}_p)}{\sum_{q=1}^{HW}\exp(W_{a,c} \tilde{F}_q)},
\qquad a_c\in\Delta^{HW-1}.
\label{eq:attn}
\end{equation}
The raw descriptor for class $c$ is defined as the \emph{expectation} of the gated features under this learned distribution:
\begin{equation}
z_c=\sum_{p=1}^{HW}a_{c,p}\,\tilde{F}_p=\mathbb{E}_{p\sim a_c}\!\left[\tilde{F}_p\right]\in\mathbb{R}^{D}.
\label{eq:pool}
\end{equation}
These vectors are stacked into the initial descriptor matrix $Z=[z_1,\dots,z_C]^{\!\top}\in\mathbb{R}^{C\times D}$. 

This formulation provides a rigorous bound on spatial focus via the Shannon entropy $\mathcal{H}(a_c)\in[0,\log HW]$. Crucially, at the maximum-entropy limit ($a_{c,p}=1/HW$), expectation pooling exactly recovers standard GAP. Thus, the model acts as a soft, learnable selector: it can revert to a shared global descriptor when localization is uninformative, but sharpens toward disease-specific regions when distinct anatomical signatures are required.

\subsection{Clinical Prior Graph Construction}
\label{ssec:grap}
Because thoracic findings are physiologically coupled, the isolated descriptors $Z$ must exchange context. We achieve this by treating the classes as nodes in a directed graph. We capture dataset-level label dependencies via the empirical conditional co-occurrence matrix $M\in[0,1]^{C\times C}$, computed once from the training labels:
\begin{equation}
M_{ij}=\frac{\sum_{n} y^{(n)}_i y^{(n)}_j}{\sum_{n} y^{(n)}_j}
       \approx P(y_i=1\mid y_j=1).
\label{eq:cooc}
\end{equation}
By Bayes' theorem, $M_{ij}P(y_j)=M_{ji}P(y_i)$. Consequently, $M$ is inherently asymmetric, correctly reflecting the clinical reality that a rare finding may depend strongly on a common one without the converse being true. 

The graph adjacency $A\in\Delta^{C-1}$ is parameterized by a learnable matrix $\Theta\in\mathbb{R}^{C\times C}$, row-normalized via softmax. To inject the clinical prior, we \emph{warm-start} the edge weights from the empirical co-occurrence:
\begin{equation}
A_{ij}=\frac{\exp(\Theta_{ij})}{\sum_{k}\exp(\Theta_{ik})},
\qquad \Theta^{(0)}=\log(M+\epsilon),
\label{eq:adj}
\end{equation}
with small $\epsilon>0$. Row-normalizing $M$ turns it into a stochastic matrix, allowing us to interpret the graph traversal as a Markov chain over diseases, where evidence is borrowed via a one-step random walk.

\subsection{Clinical Co-occurrence Message Passing}
\label{ssec:message-passing}
Node representations are updated via one step of residual message passing:
\begin{equation}
z'_i= z_i+\sum_{j=1}^{C}A_{ij}\,W_g\, z_j,
\qquad W_g\in\mathbb{R}^{D\times D}.
\label{eq:mp}
\end{equation}
Here, $W_g$ (constrained via spectral normalization to ensure $\lVert W_g\rVert_2<1$) projects all class descriptors into a common message space before aggregation. Stacking the nodes into $\zeta=[z_1^{\!\top},\dots,z_C^{\!\top}]^{\!\top}\in\mathbb{R}^{CD}$, Eq.~\eqref{eq:mp} can be written as the linear operator $\zeta'=(I_{CD}+A\otimes W_g)\zeta$. The stability of this coupling is guaranteed by the following property:

\begin{proposition}[Stability of coupling]
\label{prop:stab}
Let $A$ be row-stochastic. Every eigenvalue $\eta$ of the propagation operator $P=I_{CD}+A\otimes W_g$ satisfies $\lvert\eta-1\rvert\le\lVert W_g\rVert_2$, where $\lVert\cdot\rVert_2$ is the spectral norm. If $\lVert W_g\rVert_2<1$, $P$ is invertible and the update is a bounded, well-conditioned perturbation of the identity.
\end{proposition} 
\noindent\emph{Proof.} By definition, the eigenvalues of $A\otimes W_g$ are the pairwise products $\lambda_\ell(A)\mu_k(W_g)$. Row-stochasticity implies, by the Perron--Frobenius theorem~\cite{horn2012matrix}, that $\lvert\lambda_\ell(A)\rvert\le1$. Hence $\lvert\lambda_\ell\mu_k\rvert\le\lvert\mu_k\rvert\le\lVert W_g\rVert_2$, and the eigenvalues of $P=I+A\otimes W_g$ are $1+\lambda_\ell\mu_k$, giving $\lvert\eta-1\rvert\le\lVert W_g\rVert_2$. \hfill$\square$
\\

\noindent This ensures that a class never loses its own spatial evidence; the injected clinical context acts strictly as a bounded refinement.

\subsection{Bilinear Decoding and Objective Function}
\label{ssec:decode}
Each refined node $z'_i$ is scored by an independent linear classifier $w_i\in\mathbb{R}^{D}$ (with bias $b_i$) to produce the final probability $\hat y_i=\sigma(\ell_i+b_i)$. Substituting Eq.~\eqref{eq:mp}, the logit $\ell_i$ decomposes exactly as shown in Fig.~\ref{fig:arch}E:
\begin{equation}
\ell_i + b_i = \underbrace{z_i^\top w_i + b_i}_{\text{own spatial evidence}} \;+\; \underbrace{\sum_{j=1}^{C}A_{ij}\,z_j^\top\!\big(W_g^\top w_i\big)}_{\text{clinical context evidence}}.
\label{eq:bilinear}
\end{equation}
Equation~\eqref{eq:bilinear} reveals the precise mechanism of C$^2$A: the logit is the sum of the class's own spatial evidence and a context term, formulated as a sum of \emph{bilinear interactions} between neighbor $j$'s descriptor and class $i$'s classifier, gated by the learned clinical edge $A_{ij}$. The entire parameter set $\Omega = \{\Phi, \text{GCG}, W_a, W_g, \Theta, \{w_i, b_i\}_{i=1}^C\}$ is optimized end-to-end via multi-label Binary Cross-Entropy (BCE) loss. For an image $x$ with ground-truth labels $y$, the objective is formulated as:
\begin{equation}
\mathcal{L}(\Omega)=-\frac{1}{C}\sum_{i=1}^{C}\!
\Big[y_i\log\hat{y}_i+(1-y_i)\log(1-\hat{y}_i)\Big].
\label{eq:bce}
\end{equation}
Because the graph parameters $\Theta$ remain differentiable through the softmax Jacobian ($\partial A_{ij}/\partial\Theta_{ik}=A_{ij}(\delta_{jk}-A_{ik})$), the warm-start is not a hard constraint. If visual evidence consistently contradicts the statistical prior, gradient descent dynamically adjusts the clinical edges to minimize Eq.~\eqref{eq:bce}. Finally, the C$^2$A head is highly parameter-efficient: the co-occurrence coupling adds only one $D\!\times\!D$ projection ($W_g$) and $C^2$ scalars ($\Theta$) beyond the attention parameters, representing a negligible $\mathcal{O}(D^2 + C^2)$ overhead compared to the backbone.

\begin{figure*}[t]
\centering
\vspace{-3mm} 
\begin{tabular}{c}
\includegraphics[width=0.95\textwidth]{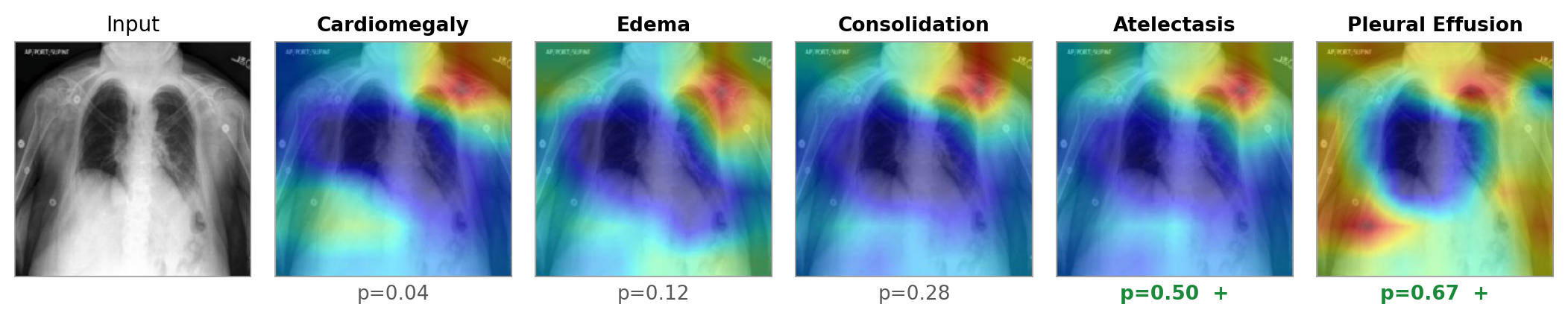} \\
{\small (a) C$^2$A per-class attention maps for a single multi-label study.} \\[1.5mm]
\includegraphics[width=\textwidth]{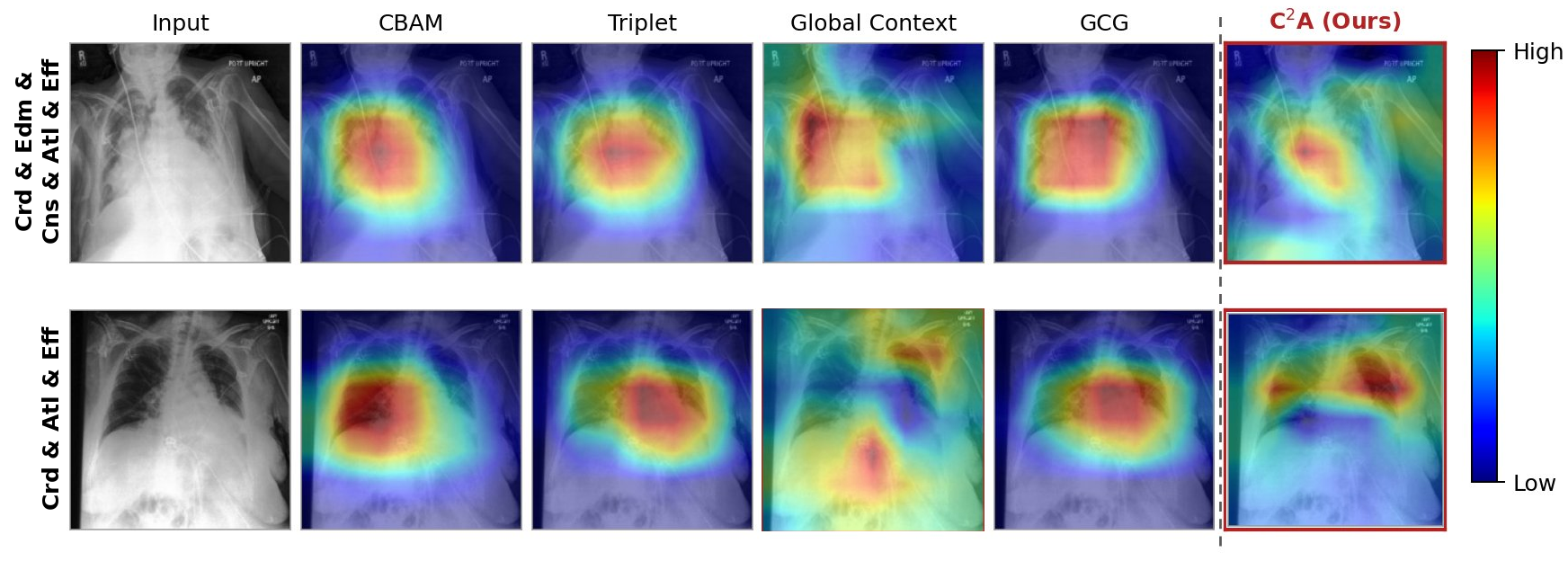} \\
{\small (b) Comparison of C$^2$A against globally-pooled baselines across different studies.}
\end{tabular}
\vspace{-2mm} 
\caption{\textbf{Qualitative evaluation via Grad-CAM++.} (a) C$^2$A generates distinct spatial support for concurrent pathologies. (b) Globally-pooled baselines exhibit severe center bias, whereas C$^2$A correctly localizes peripheral and focal lung opacities.}
\label{fig:gradcam}
\vspace{-4mm} 
\end{figure*}

\section{Experiments and Results}
\label{sec:exp}

\subsection{Dataset and Clinical Setup}
\label{ssec:data}
We evaluate the proposed architecture on CheXpert~\cite{irvin2019chexpert}, a large-scale public corpus comprising $223{,}414$ training chest radiographs. Following standard evaluation protocols, we focus on the five competition tasks selected for their clinical severity and high prevalence: Cardiomegaly, Edema, Consolidation, Atelectasis, and Pleural Effusion. We adopt the standard U-Ones policy (mapping uncertain to positive) to encourage high sensitivity; ablation with U-Zeros yielded consistent relative rankings. 

The empirical conditional co-occurrence matrix $M_{ij}=P(\text{row}\,|\,\text{col})$, computed over the training set and used to warm-start our graph adjacency, is ordered by Cardiomegaly, Edema, Consolidation, Atelectasis, and Pleural Effusion:
\begin{equation}
M = 
\begin{bmatrix}
1.00 & 0.26 & 0.15 & 0.16 & 0.17 \\
0.49 & 1.00 & 0.26 & 0.29 & 0.34 \\
0.18 & 0.17 & 1.00 & 0.41 & 0.23 \\
0.30 & 0.30 & 0.64 & 1.00 & 0.37 \\
0.48 & 0.51 & 0.54 & 0.54 & 1.00
\end{bmatrix}
\label{eq:cooc_matrix}
\end{equation}
This matrix reveals strong, asymmetric clinical dependencies. For instance, Atelectasis appears in $64\%$ of Consolidation studies ($M_{4,3}=0.64$), and Pleural Effusion accompanies Edema in $51\%$ of cases ($M_{5,2}=0.51$), validating our hypothesis that these targets cannot be treated as statistically independent.

\subsection{Implementation Details}
\label{ssec:impl}
All images are resized to $224\times224$ pixels and normalized using standard ImageNet channel statistics. The DenseNet-121 backbone $\Phi$ is initialized with ImageNet weights, while the C$^2$A head parameters are initialized orthogonally. Models are trained end-to-end using the Adam optimizer with a learning rate of $10^{-4}$ and a weight decay of $10^{-5}$. We use a batch size of $16$, training to convergence across three random seeds to ensure statistical significance ($p<0.05$, DeLong's test). Diagnostic performance is measured via the per-class Area Under the Receiver Operating Characteristic curve (AUROC) and the macro-mean AUROC. To rigorously assess deployment feasibility, we also report the parameter count (Millions), theoretical compute (GFLOPs), and empirical inference latency (milliseconds per image) measured on a single NVIDIA GPU.


\begin{table}[t]
\centering
\caption{Performance \& efficiency comparison on the official CheXpert validation set.}
\label{tab:main}
\resizebox{\columnwidth}{!}{%
\begin{tabular}{lccccc}
\toprule
& \multicolumn{2}{c}{\textbf{Spatial/Channel}} & \multicolumn{2}{c}{\textbf{Context/Gating}} & \textbf{Ours} \\
\cmidrule(lr){2-3} \cmidrule(lr){4-5} \cmidrule(lr){6-6}
Metric & CBAM~\cite{woo2018cbam} & Triplet~\cite{misra2021rotate} & GC~\cite{cao2020global} & GCG~\cite{cherukuri2024guided} & \textbf{C$\mathbf{^2}$A} \\
\midrule
Cardiomegaly     & 0.7417 & 0.8307 & 0.8107 & 0.8490 & \textbf{0.8520} \\
Edema            & 0.9158 & 0.9303 & 0.9255 & 0.9320 & \textbf{0.9350} \\
Consolidation    & 0.9034 & 0.8780 & 0.9095 & 0.8984 & \textbf{0.9110} \\
Atelectasis      & 0.8316 & 0.8295 & 0.8300 & 0.8232 & \textbf{0.8380} \\
Pleural Effusion & 0.9203 & 0.9242 & 0.9246 & 0.9344 & \textbf{0.9390} \\
\midrule
\textbf{Mean AUROC} & 0.8625 & 0.8786 & 0.8800 & 0.8874 & \textbf{0.8950} \\
\midrule
Params (M)       & 8.77 & \textbf{8.51} & 8.77 & 10.87 & 11.92 \\
GFLOPs           & \textbf{5.79} & 5.80 & \textbf{5.79} & 6.00 & 6.01 \\
\bottomrule
\end{tabular}%
}
\end{table}

\subsection{Quantitative Analysis}
\label{ssec:quant}
Table~\ref{tab:main} compares C$^2$A against four state-of-the-art attention baselines (CBAM~\cite{woo2018cbam}, Triplet~\cite{misra2021rotate}, Global Context~\cite{cao2020global}, and Guided Context Gating~\cite{cherukuri2024guided}). All baselines share the DenseNet-121 backbone but rely on Global Average Pooling (GAP) prior to classification.

\noindent \textbf{Diagnostic Performance.} Context/Gating mechanisms generally outperform Spatial/Channel approaches, establishing GCG as the strongest baseline. However, C$^2$A achieves the highest overall performance, yielding a macro-mean improvement of $\uparrow 0.76\%$ over GCG. Ablation confirms our design: removing the graph (class-specific pooling only) drops mean AUROC to $0.889$, while random graph initialization yields $0.891$, proving the necessity of the co-occurrence warm-start. Furthermore, gains concentrate where global pooling fails. Collapsing GCG features via GAP destroys spatial separation, causing Atelectasis performance to regress below CBAM. C$^2$A explicitly resolves this bottleneck. The clinical prior acts as a targeted regularizer, rescuing Atelectasis ($\uparrow 1.48\%$ over GCG, $\uparrow 0.64\%$ over CBAM) by borrowing spatial context from co-occurrent findings.

\noindent \textbf{Computational Efficiency.} C$^2$A improves representational power without sacrificing scalability. While the class-specific heads and graph adjacency slightly increase parameters over GCG ($11.92$M vs. $10.87$M), the theoretical compute overhead is negligible ($6.01$ vs. $6.00$ GFLOPs). Consequently, empirical inference latency remains highly competitive ($24.35$ vs. $24.12$ ms/image), ensuring C$^2$A is fully capable of real-time clinical deployment.

\subsection{Qualitative Analysis}
\label{ssec:qual}
Beyond predictive performance, C$^2$A is interpretable by construction. We validate its spatial representations using Grad-CAM++~\cite{chattopadhay2018gradcampp}. 

Figure~\ref{fig:gradcam}(a) illustrates how C$^2$A resolves the spatial entanglement typical of multi-label CXRs. Because standard models use a shared global descriptor, they often produce a single, compromised saliency map for the entire image. In contrast, C$^2$A's expectation pooling generates independent distributions per class. For a patient with concurrent findings, it correctly isolates the focal upper-lobe opacity for Atelectasis while independently highlighting the lower lung fields for Pleural Effusion, proving that the model does not sacrifice disease-specific precision. 

Figure~\ref{fig:gradcam}(b) exposes a critical failure mode of standard architectures: \textbf{center bias}. When baselines (e.g., CBAM, GCG) collapse feature maps via Global Average Pooling, they tend to anchor their attention on the cardiac silhouette or mediastinum, even for lung-specific pathologies like Consolidation. C$^2$A structurally prevents this. By preserving spatial resolution and routing descriptors through the co-occurrence graph, C$^2$A breaks the center bias and shifts its focus to the anatomically correct peripheral lung fields. Furthermore, for diffuse conditions like Edema, C$^2$A's saliency naturally encompasses both the lungs and the heart. This visually confirms the mechanism of our graph: the model actively borrows spatial context from related nodes (e.g., Cardiomegaly) to inform its diagnosis.

The quality of these visualizations highlights the translational value of C$^2$A. By explicitly mapping the bilinear interaction between spatial evidence and co-occurrence context (Eq.~\ref{eq:bilinear}), C$^2$A provides a transparent reasoning chain. Radiologists can verify not only \emph{what} the model predicts, but \emph{where} it is looking and \emph{which} related findings influenced the decision. 

\section{Conclusion}
\label{sec:conc}
We introduced C$^2$A, a lightweight classification head unifying spatial localization and clinical label dependencies. By replacing global average pooling with class-specific expectation pooling and coupling the descriptors via a warm-started co-occurrence graph, C$^2$A achieves a superior $0.895$ macro-mean AUROC on CheXpert. Crucially, rather than lifting all tasks uniformly, the clinical prior specifically rescues pathologies with ambiguous spatial evidence (e.g., Atelectasis, $+1.5$ points over GCG) by borrowing context from co-occurrent findings. This validates our bilinear decoding formulation: explicitly modeling co-occurrence acts as an interpretable regularizer that redistributes network capacity toward diseases with weak stand-alone evidence. C$^2$A achieves these targeted gains at the negligible cost of one linear projection and a $C\!\times\!C$ adjacency matrix, offering a scalable, clinically grounded blueprint for multi-label medical image analysis, easily scalable to the full 14-class setting.

\bibliographystyle{IEEEbib}
\bibliography{refs}

\end{document}